\pdfoutput=1

\documentclass[11pt]{article}

\usepackage{ACL2023}

\usepackage{times}
\usepackage{latexsym}
\usepackage{xspace}
\usepackage{graphicx}
\usepackage{amssymb}
\usepackage{amsmath}
\usepackage[ruled,linesnumbered]{algorithm2e}
\usepackage{graphicx}
\usepackage{standalone}
\usepackage{arydshln}
\usepackage{booktabs}
\usepackage{array}
\usepackage{makecell}
\usepackage{multirow}

\usepackage{float}
\usepackage{bbm}
\usepackage{mathabx}

\usepackage{amsthm}
\newtheorem{theorem}{Theorem}[section]

\newtheorem{hypothesis}{Hypothesis}[section]
\theoremstyle{definition}
\newtheorem{definition}{Definition}[section]

\usepackage{cleveref}
\usepackage[T1]{fontenc}

\usepackage[utf8]{inputenc}

\usepackage{microtype}

\usepackage{inconsolata}
\usepackage{soul}

\newcommand{\eg}{\emph{e.g.,}\xspace}
\newcommand{\ie}{\emph{i.e.,}\xspace}
\newcommand{\myref}[1]{Eq.\ref{#1}}

\title{Towards Higher Pareto Frontier in Multilingual Machine Translation}

\author{Yichong Huang$^{\dag}$, Xiaocheng Feng$^{\dag \ddag}$, Xinwei Geng$^{\dag}$, Baohang Li$^{\dag}$, Bing Qin$^{\dag \ddag}$\\
  $^{\dag}$Harbin Institute of Technology\quad \quad \quad $^\ddag$ Peng Cheng Laboratory\\
  \texttt{\{ychuang,xcfeng,xwgeng,baohangli,qinb\}@ir.hit.edu.cn} \\}

\begin{document}
\maketitle

\begin{abstract}
Multilingual neural machine translation has witnessed remarkable progress in recent years. 
However, the long-tailed distribution of multilingual corpora poses a challenge of Pareto optimization, \ie optimizing for some languages may come at the cost of degrading the performance of others.
Existing balancing training strategies are equivalent to a series of Pareto optimal solutions, which trade off on a Pareto frontier\footnote{In Pareto optimization, Pareto optimal solutions refer to solutions in which none of the objectives can be improved without sacrificing at least one of the other objectives. The set of all Pareto optimal solutions forms a Pareto frontier.}.
In this work, we propose a new training framework, Pareto Mutual Distillation (Pareto-MD), towards pushing the Pareto frontier outwards rather than making trade-offs.
Specifically, Pareto-MD collaboratively trains two Pareto optimal solutions that favor different languages and allows them to learn from the strengths of each other via knowledge distillation.
Furthermore, we introduce a novel strategy to enable stronger communication between Pareto optimal solutions and broaden the applicability of our approach. 
Experimental results on the widely-used WMT and TED datasets show that our method significantly pushes the Pareto frontier and outperforms baselines by up to +2.46 BLEU\footnote{Our code is publicly available at \url{https://github.com/OrangeInSouth/Pareto-Mutual-Distillation}}. 
\end{abstract}

\section{Introduction}
Multilingual neural machine translation (MNMT) is a popular paradigm that uses a unified model to handle the entire translation process for multiple language pairs~\cite{ha-etal-2016-toward,firat-etal-2016-multi,johnson-etal-2017-googles}.
This paradigm is particularly effective at improving the performance of low-resource languages through transfer learning~\cite{aharoni-etal-2019-massively,dabre2020survey,siddhant2022towards}. 
Besides, MNMT is highly deployable since only one model is required~\cite{fan2021beyond,yang-etal-2021-multilingual-machine,https://doi.org/10.48550/arxiv.2207.04672}.

\begin{figure}[ht]
  \centering
    \includegraphics[clip,width=1.0\columnwidth,]{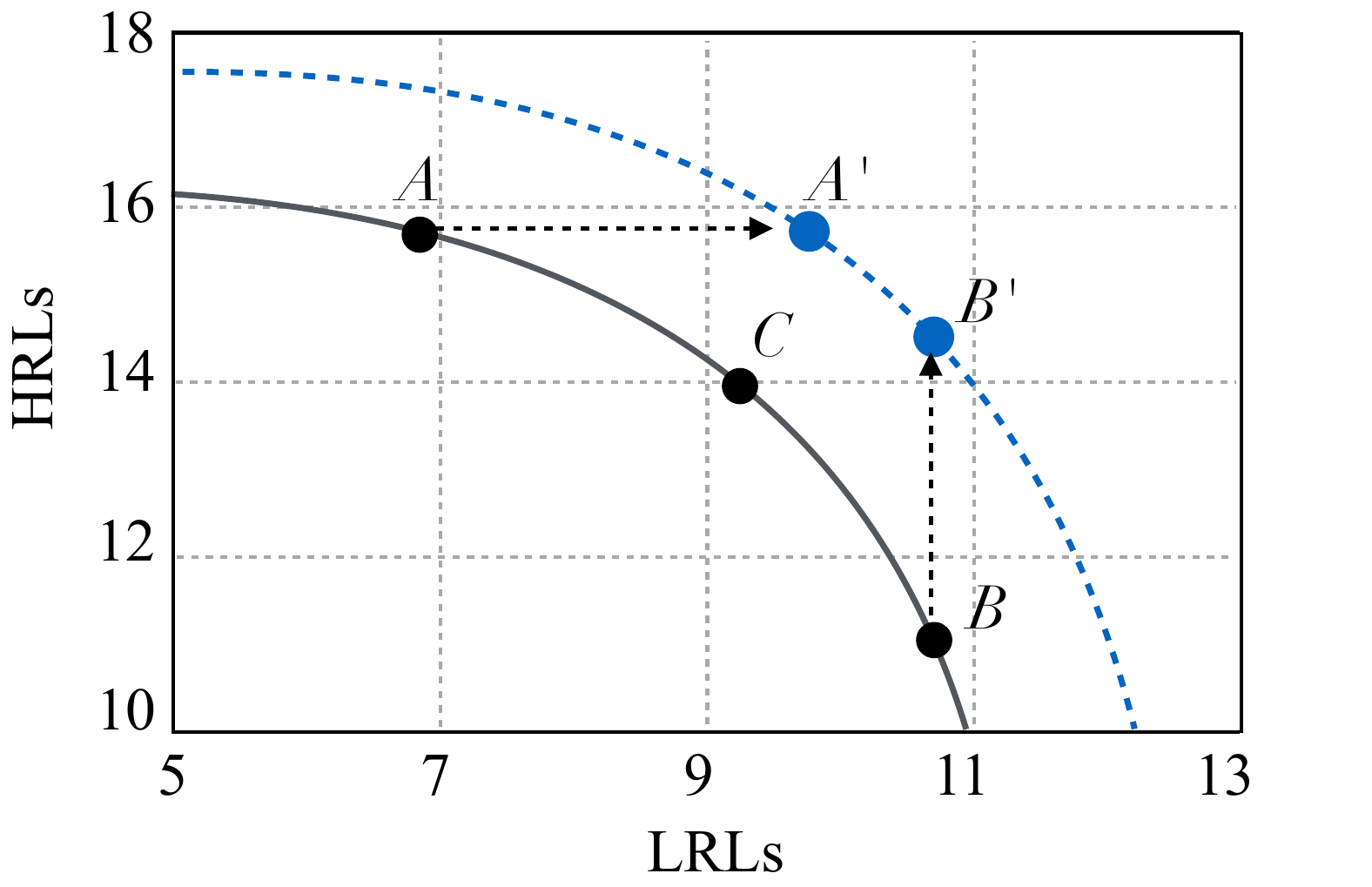}
    \caption{Multilingual performance frontier shifts outwards. X-axis and Y-axis indicate the performance of Low-Resource Languages and High-Resource Languages, respectively. Existing methods reflect a trade-off on the Pareto frontier (\ie the gray curve). Our work aims to \textit{push the original Pareto frontier} \ie the blue dotted curve. To this effect, we ameliorate each individual model's shortcoming while retaining their strengths, \eg moving right the solution $A$ to $A'$ and moving up the solution $B$ to $B'$, via our Pareto Mutual Distillation.}
  \label{fig:motivation}
\end{figure}

However, the severely imbalanced distribution of multilingual training data puts the MNMT in a situation of Pareto optimization (also known as multi-objective optimization). That is, when some languages are optimized, others degenerate.
Existing methods can be considered a set of Pareto optimal solutions that trade off on a Pareto frontier, which focus on balancing the performance across different languages by adjusting the sampling distribution ~\citep{arivazhagan-etal-2019-massively,wang-etal-2020-balancing,wu-etal-2021-uncertainty}.
The widely-used temperature-based sampling~\citep{arivazhagan-etal-2019-massively} is typical evidence of the claim above, which uses a hyper-parameter to smooth the training distribution over all language pairs to enhance the representation of low-source Languages (LRLs) while sacrificing the which of High-Resource Languages (HRLs).
Despite the emergence of several sophisticated dynamic sampling technologies designed to overcome the inflexibility of temperature-based sampling, their performance remains restricted to this Pareto frontier~\citep{wang-etal-2020-balancing,zhou-etal-2021-distributionally,zhang-etal-2021-competence-based}.

In this work, we propose a novel training framework, named Pareto Mutual Distillation (Pareto-MD), to push the Pareto frontier of multilingual models.
Specifically, Pareto-MD uses different training distributions that favor dissimilar subsets of languages to train two multilingual models simultaneously.
These two models learn from each other at each training step with knowledge distillation.
The underlying idea of Pareto-MD is to address shortcomings of individual Pareto optimal solutions via access to a better one in terms of that shortcoming, thereby raising the Pareto frontier, as Fig.~\ref{fig:motivation} depicts.
To fully exploit the potential of our approach in multilingual settings, we further propose Automatic Pareto Mutual Distillation, which dynamically determines the contribution of distillation learning loss on each objective. 
These contributions, controlled by a set of distillation weights, adapt automatically to the evolving models, eliminating the need for manual hyper-parameter search.

While our method applies essentially to any multi-objective optimization problem, we specifically demonstrate its benefit on multilingual machine translation.
The experimental results on two widely-used datasets demonstrate the effectiveness of our method, which improves up to +2.46 BLEU, and the further analysis shows the Pareto frontier is pushed outwards visibly.

\section{Preliminaries}
Neural machine translation (NMT) is a classic NLP task that translates a sentence $x$ in source language into a sentence $y$ in target language~\citep{kalchbrenner-blunsom-2013-recurrent,sutskever-etal-2014-sequence,bahdanau-etal-2015-neural,vaswani-etal-2017-attention}.
Given a parallel corpus ${D} = \{(x,y) \in \mathcal{X} \times \mathcal{Y}\}$, the NMT model is commonly trained by minimizing the negative log-likelihood loss:
\begin{align}
\label{eq:NLL_loss}
\mathcal{L}_{ce} =  \sum_{(x, y) \, \sim D} \sum\limits_{i \le |y|} \ - \log p(y_i | x,y_{<i}; \theta),
\end{align}
 where $p(\cdot|\cdot; \theta)$ maps the source sentence and previous generated text to the next target token.

\subsection{Multilingual Machine Translation}
Given a set of language pairs $L$, the MNMT model is trained on the combination of $|L|$ parallel datasets: $\{D^{train}_{\ell} \}_{\ell=1}^{|L|}$, where $D^{train}_\ell$ is the dataset of language pair $(S_\ell,T_\ell)$. 
In order to encode and decode the text in various languages into and from a universal semantic space, a large multilingual vocabulary $\mathcal{V}$ is constructed.
The language tag is appended to the beginning of source sentences as a hint of the target language.
The MNMT model is also trained with the loss function as \myref{eq:NLL_loss} over the multilingual datasets.

\paragraph{Temperature-based Sampling.}
The multilingual datasets form a distribution $P$, where $P(\ell) = \frac{N_\ell}{\sum_j N_j}$ is the sampling probability of language pair $\ell$ and we denote the dataset size of $D^{train}_\ell$ by $N_\ell$.
Since sampling probabilities of LRLs are substantially lower than those of HRLs, the optimization towards LRLs can be overwhelmed by those of HRLs.
To resolve this issue, \citet{arivazhagan-etal-2019-massively} propose temperature-based sampling, introducing a hyper-parameter $\tau$ to re-scale the smoothness of training distribution.
Concretely, the sampling probability of each language pair $\ell$ is set to:
\begin{equation}
\label{eq:temperate-sampling}
P(\ell) = \frac
    {
        N_\ell^{1/\tau}
    }
    {
        \sum_j N_j^{1/\tau}
    },    
\end{equation}
increasing the value of $\tau$ produces smoother training distributions and stronger preferences on LRLs.

\subsection{Mutual Distillation}
Knowledge Distillation (KD) is a popular technology for knowledge transfer, which originates from compressing a static high-capacity model (teacher model) into a small compact model (student model)~\citep{hinton2015distilling}.
Mutual distillation is a variant of KD~\citep{zhang2018deep,guo2020online}. 
Instead of using a pre-trained teacher model, \textbf{mutual distillation involves training more than one model simultaneously}, with each model teaching the other throughout the training process.
Mutual distillation takes the same loss function as vanilla knowledge distillation, that is:
\begin{equation}
\begin{aligned}
    \mathcal{L}_{kd} =
    \sum_{i \le |y|} \sum_{w \in \mathcal{V}}
    -& \, p(w | x, y_{<i}; \theta^T) \\
    \cdot \log \, &p(w | x, y_{<i}; \theta^S),
\end{aligned}
\end{equation}
where $\mathcal{V}$ is the target-side vocabulary, $\theta^S$ and $\theta^T$ are the student model and teacher model.
\textit{The major difference of Pareto-MD from vanilla mutual distillation is that we train two models with different sampling distributions to make them favor different sets of objectives.}
\section{Pareto Mutual Distillation}
\begin{figure}[t]
  \centering
    \includegraphics[clip,width=1.0\columnwidth,]{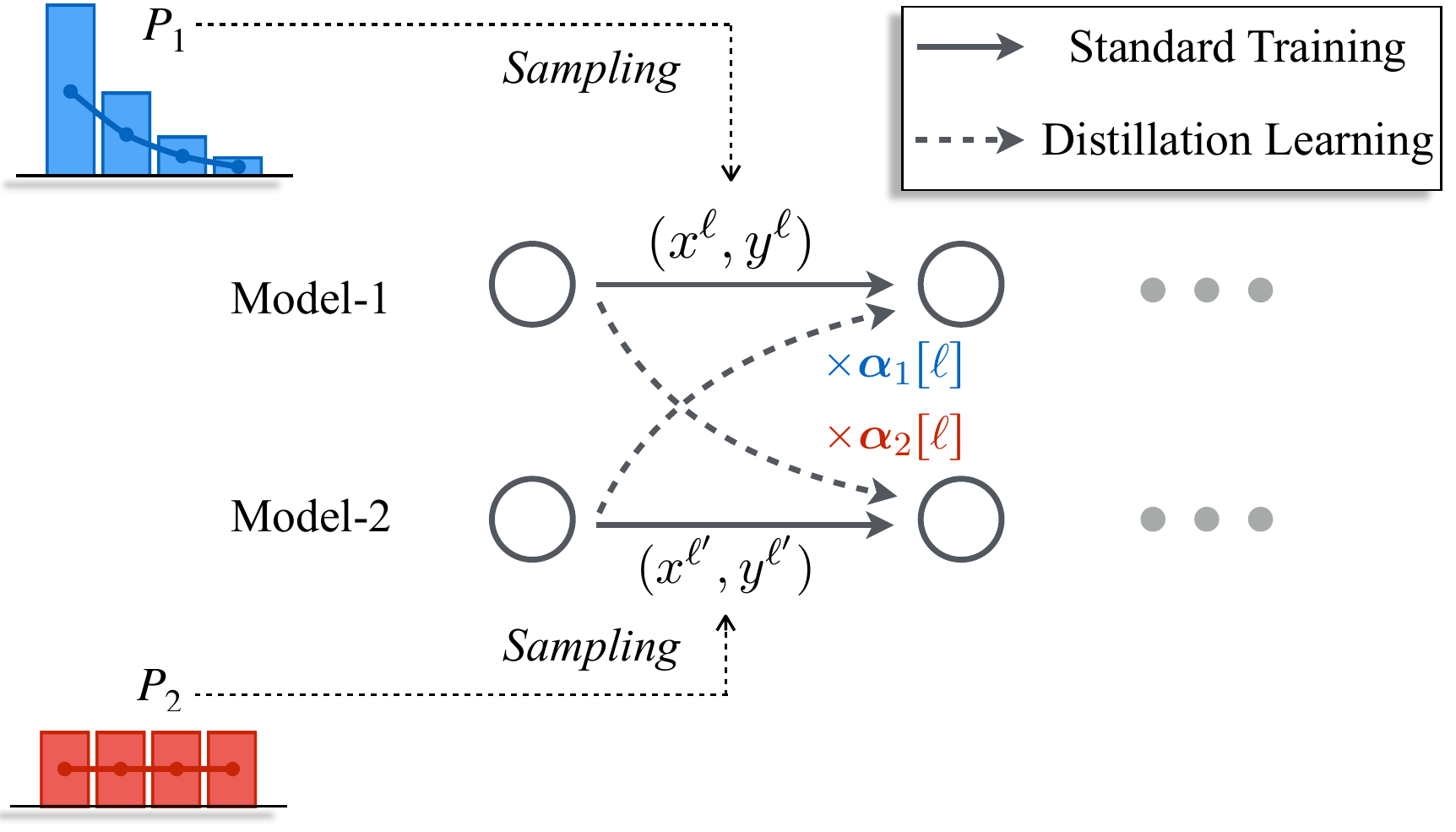}
    \caption{Illustration of Pareto-MD, using different sampling distributions to train two models. 
    At each step, both models additionally mimic the output of each other via knowledge distillation. 
    The distillation learning of each model is weighted by language-specific distillation weights $\boldsymbol{\alpha}_i[\ell]$ deduced with specific strategies.}
    \label{fig:Pareto-MD}
\end{figure}
In this section, we first introduce our training framework Pareto-MD (\S\ref{subsection:method framework}). 
Next, two strategies that determine the important distillation weights, \textsc{Uni}-PMD and \textsc{Bi}-PMD, are shown (\S\ref{subsection:UniMD_BiMD}).
To overcome the flaws of these two strategies above, \textsc{Auto}-PMD is further proposed (\S\ref{subsection:auto-momd}).

\subsection{Framework}
\label{subsection:method framework}
We illustrate our Pareto-MD in Fig.~\ref{fig:Pareto-MD}.
Pareto-MD simultaneously trains two models, denoted by $\theta_1$ and $\theta_2$, using different sampling distributions, $P_1$ and $P_2$, that make each model favor a different set of language pairs. 
To obtain expected distributions, we adopt temperature-based sampling, as shown in ~\myref{eq:temperate-sampling}, and set $\tau=1$ for $P_1$, $\tau>1$ (\eg $\tau=5$ commonly) for $P_2$.
In this way, $\theta_1$ prefers HRLs, and $\theta_2$ prefers LRLs.

At each training step, for each model $\theta_i$ where $i \in \{1, 2\}$, Pareto-MD first draws a language pair $\ell$ from training distribution $P_i$, then a mini-batch of sentence pairs $B_{\ell} = \{x_{\ell}, y_{\ell} \}$ are sampled from $D^{train}_{\ell}$.
Next, the model $\theta_i$ is trained to fit $B_{\ell}$ and match the output of the other model, \ie $\theta_{3-i}$.
The overall loss function for model $\theta_i$ is defined as:
\begin{equation}
\label{eq:Pareto-MD_Loss}
\begin{aligned}
\mathcal{L}_{PMD} = 
(1 - \boldsymbol{\alpha}_i[\ell]) \,\times\, &\mathcal{L}_{ce}(B_l; \theta_i)  \\
+\  \boldsymbol{\alpha}_i[\ell] \ \ \times\, &\mathcal{L}_{kd}(B_{\ell}; \theta_i, \theta_{3-i}),
\end{aligned}     
\end{equation}
where $\boldsymbol{\alpha}_i \in \mathbb{R}^{|L|}$ is the multilingual distillation weight vector of $\theta_i$ and $\boldsymbol{\alpha}_i[\ell] \in [0, 1]$ is the distillation weight for language pair $\ell$. 
$\boldsymbol{\alpha}_i[\ell]$ is crucial as controlling the extent how much $\theta_i$ should learn from $\theta_{3-i}$ in direction $\ell$. 
When $\boldsymbol{\alpha}_i[\ell] = 0$, $\theta_i$ does not acquire information from $\theta_{3-i}$ in $\ell$.
The values of $\boldsymbol{\alpha}_i$ are determined by the specific strategy.
We summarize the whole training framework in Alg.\ref{alg:Pareto-MD}.

\setlength{\textfloatsep}{10pt}
\begin{algorithm}[t]

\DontPrintSemicolon
\SetInd{0.5em}{0.7em} 
\SetKwInput{KwIn}{Input\quad\ }
\SetKwInput{KwInit}{Initialize}
\small
\caption{Pareto-MD}\label{alg:Pareto-MD}
\KwIn{Datasets $\{D^{train}_{\ell} \}_{\ell=1}^{|L|}$, two training distributions $P_1, P_2$, learning rate $\eta$, distillation weights updating strategy $\mathcal{S}$, updating interval $\mathcal{T}$} 
    
 \KwInit{Randomly initialize model $\theta_1$ and $\theta_2$, set multilingual distillation weights $\boldsymbol{\alpha}_1,\boldsymbol{\alpha}_2 = \boldsymbol{0}$, training step $t = 0$ }
    
\While{not converged}{
    $t \gets t+1$ \;
    \For{$i \gets 1 \ to\  2$}{
        Sample a language pair $\ell$ from $P_i$ \;
        Draw a batch of samples $B_{\ell}$ from $D^{train}_{\ell}$ \;
        $\theta_i \gets \theta_i - \eta \nabla_{\theta_i}\mathcal{L}_{PMD}(B_{\ell};\theta_i, \theta_{3-i}, \boldsymbol{\alpha}_i[\ell])$
    }

    \If{$t \ \% \ \mathcal{T} = 0$}{
        Update $\boldsymbol{\alpha}_1, \boldsymbol{\alpha}_2$ with the specific strategy $\mathcal{S}$
    }
}

\end{algorithm}

\subsection{\textsc{Uni}-PMD and \textsc{Bi}-PMD}
\label{subsection:UniMD_BiMD}
Multilingual distillation weights $\boldsymbol{\alpha}_i$ play important roles in Pareto-MD. We present two strategies, unidirectional Pareto mutual distillation (\textsc{Uni}-PMD) and bidirectional Pareto mutual distillation (\textsc{Bi}-PMD), for determining the values of $\boldsymbol{\alpha}_i$ based on different design philosophies.

\begin{figure*}[t]
  \centering
    \includegraphics[clip,width=2.0\columnwidth,]{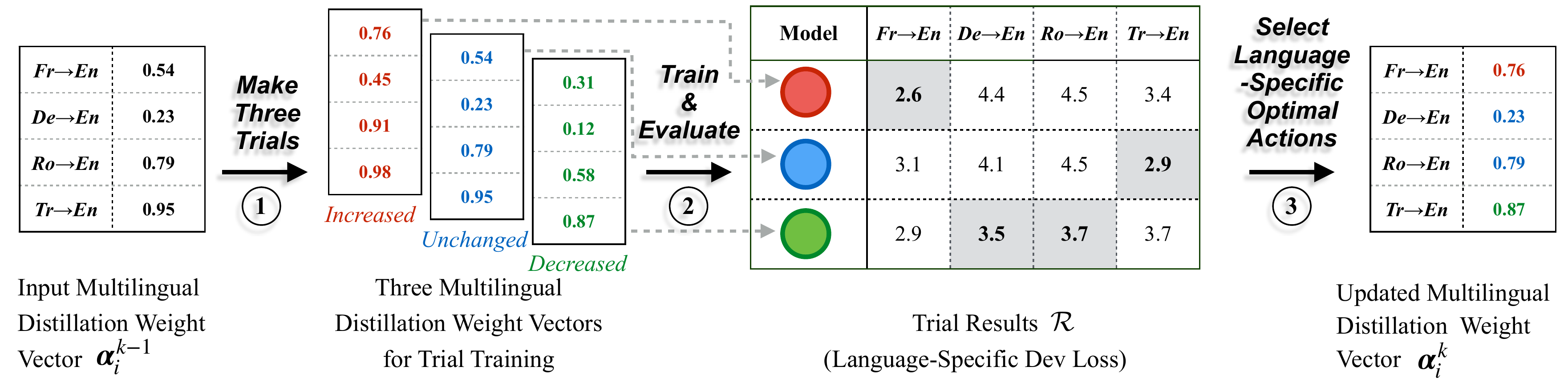}
    \caption{Process of \textsc{Auto}-PMD updating the distillation weights.
    At the $k$-th update, \textsc{Auto}-PMD makes three trials that perform three actions to all language pairs' weights and then train the current model.
    Finally, the language-specific optimal actions are selected to update the previous weights.
    Note that the value of each weight will change by different magnitudes when increased or decreased due to the non-linear nature of sigmoid function.}
    \label{fig:Auto-PMD}
\end{figure*}
\paragraph{\textsc{Uni}-PMD.} 
\textsc{Uni}-PMD is designed based on the intuition that \textit{each model should only learn from the strengths and avoid mimicking the shortcomings of the other model}.
Therefore, in each translation direction $\ell$, \textsc{Uni}-PMD lets the model that performs less well, denoted by $\theta_{\ell}^{worse}$, be distilled by the model that performs better in this direction, denoted by $\theta_{\ell}^{better}$, via setting a positive distillation weight.
Conversely, \textsc{Uni}-PMD zeros the weight to forbid $\theta_{\ell}^{better}$ from being influenced by $\theta_{\ell}^{worse}$.

Formally, given multilingual validation datasets $\{D^{valid}_{\ell}\}_{\ell=1}^{|L|}$ and a pre-defined hyper-parameter $\alpha \in [0, 1]$, in each direction $\ell \in L$, \textsc{Uni}-PMD sets the distillation weight of $\theta_i$ as:
\begin{equation}
\label{eq:uni-momd}
\begin{aligned}
\boldsymbol{\alpha}_i[\ell] = \alpha \times \mathbbm{1}\{ i = \mathop{\arg\max}\limits_{j \in \{1, 2\}} \mathcal{L}_{ce}(D^{valid}_l; \theta_j) \},
\end{aligned}
\end{equation}
where the $\mathbbm{1}\{\cdot\}$ is an indicator function, indicating whether the model $\theta_i$ performs less well on the translation of $\ell$. \textsc{Uni}-PMD updates the distillation weights every $\mathcal{T}$ steps.

\paragraph{\textsc{Bi}-PMD.} Besides, we design another strategy \textsc{Bi}-PMD based on the hypothesis that among the two models that are trained with Pareto-MD, in each translation direction $\ell$, $\theta_{\ell}^{worse}$ is also possible to improve $\theta_{\ell}^{better}$ via knowledge distillation.
This hypothesis is motivated by the recently proposed theoretical framework of \textit{Multi-View Data}~\citep{https://doi.org/10.48550/arxiv.2012.09816,he2021feature}, which theoretically reveals that each well-trained network only captures a different subset of relevant features, limiting their generalization. 
The mechanism of knowledge distillation is to help one model to learn the relevant features of another model.

The discovery motivates us to suspect that $\theta_{\ell}^{worse}$ can also improve $\theta_{\ell}^{better}$ using distillation, as $\theta_{\ell}^{worse}$ may possess relevant features that $\theta_{\ell}^{better}$ lacks.
Therefore, \textsc{Bi}-PMD allows $\theta_{\ell}^{worse}$ to affect $\theta_{\ell}^{better}$ in direction $\ell$.
Our implementation is simple: \textsc{Bi}-PMD sets all distillation weights to a positive value. Formally, given a hyper-parameter $\alpha$, the distillation weight of $\theta_i$ in direction $\ell$ is:
\begin{equation}
\label{eq:bi-momd}
\begin{aligned}
\boldsymbol{\alpha}_i[\ell] = \alpha,
\end{aligned}
\end{equation}
meaning that each model affects the other equally.

\subsection{\textsc{Auto}-PMD} \label{subsection:auto-momd}
\paragraph{Desiderata.} Both \textsc{Uni}-PMD and \textsc{Bi}-PMD determine the distillation weights of all translation directions based on a pre-defined hyper-parameter $\alpha$, which dissatisfies the following three expected properties of distillation weights: 1)~\textbf{Language-Adaptability}: the optimal distillation weights for different language pairs vary. However, the current strategies set a uniform weight for all language pairs, resulting in sub-optimal performance; 2)~\textbf{Dynamics}: existing research on mutual distillation uses a fixed distillation weight throughout the training process, which fails to adapt to the evolving models; 3)~\textbf{Generality}: it is empirically discovered that the optimal value of distillation weight varies across different datasets, incurring the extra cost of the manual hyper-parameter search.
To satisfy these three properties, we propose Automatic Pareto Mutual Distillation (\textsc{Auto}-PMD) to automatically decide the value of each direction’s distillation weight according to training dynamics.

\paragraph{Approach.} \textsc{Auto}-PMD updates multilingual distillation weight vector $\boldsymbol{\alpha}_i$ every $\mathcal{T}$ steps. 
We denote the values of $\boldsymbol{\alpha}_i$ after the $k$-th update by $\boldsymbol{\alpha}^k$.
Note that the subscript $i$ of $\boldsymbol{\alpha}_i$ is omitted for clarity.
The update process is modeled as Markov Chain~\citep{norris1998markov}. 
All distillation weights are initialized at the beginning of training as a small value, \ie $\boldsymbol{\alpha}^0[\ell]=0.1$.  
Three actions on distillation weight are defined:
\begin{equation}
\begin{aligned}\label{eq: actions set}
\mathcal{F} = \{ f_{\uparrow}(\cdot), f_{\downarrow}(\cdot), f_{=}(\cdot) \},
\end{aligned}
\end{equation}
which aim to increase, decrease and keep the value of distillation weight unchanged. 
At the $k$-th update, \textsc{Auto}-PMD decides the values of $\boldsymbol{\alpha}^{k}$ according to the previous state $\boldsymbol{\alpha}^{k-1}$.
We exemplify the process of each update step in Fig.~\ref{fig:Auto-PMD} and precisely describe it in Alg.~\ref{alg:Auto-PMD}. As illustrated in Fig.~\ref{fig:Auto-PMD}, the update process is divided into three steps.

In the first step, given the previous distillation weights $\boldsymbol{\alpha}^{k-1}$, \textsc{Auto}-PMD makes three trials, generating three multilingual distillation weight vectors for the trial training of the next step. Each vector is obtained by performing an action (\eg increasing) on all values of $\boldsymbol{\alpha}^{k-1}$. These three vectors, corresponding to three colorful vectors in Fig.~\ref{fig:Auto-PMD}, form a set which is referred to as search space $\widetilde{O}^k$. In fact, the trial training of next step should be conducted over the entire search space $O^k$, which is the Cartesian product of possible subsequent states of each language-specific distillation weight $\boldsymbol{\alpha}^{k-1}[\ell]$:
\begin{equation}
\label{eq:complete search space}
O^k = \mathop{\bigtimes}\limits_{\ell \in L} \{ f(\boldsymbol{\alpha}^{k-1}[\ell]) \,|\, f \in \mathcal{F} \}.
\end{equation}
However, this search space grows exponentially as the number of languages increases, that is, $|O^k| = |\mathcal{F}|^{|L|}$. 
To overcome the non-trivial cost, the sub-space $\widetilde{O}^{k}$ is adopted.
Furthermore, we prove that based on the \textit{Distillation Weights Independence} assumption, the optimal solution searched in $\widetilde{O}^k$ is equivalent to that of $O^k$.
The mathematical description of this assumption and the proof are demonstrated in~\S\ref{section:proof of equivalence}.

\begin{algorithm}[t]
\small
\caption{\textsc{Auto}-PMD}\label{alg:Auto-PMD}
\DontPrintSemicolon
\SetInd{0.5em}{1em} 
\SetKwInput{KwIn}{Input\quad\ }
\SetKwInput{KwOut}{Output\ \ }
\SetKwInput{KwInit}{Initialize}
\KwIn{Multilingual trial datasets $\{D^{trial}_\ell\}_{\ell=1}^{|L|}$, validation datasets $\{D^{valid}_\ell\}_{\ell=1}^{|L|}$, the training model $\theta_1$and $\theta_2$, search space $\widetilde{O}_1^k$, $\widetilde{O}_2^k$, 
    distillation weights $\boldsymbol{\alpha}^{k-1}_1, \boldsymbol{\alpha}^{k-1}_2$ } 
 \KwOut{$\boldsymbol{\alpha}^k_1,\boldsymbol{\alpha}^k_2$}
 \KwInit{Initialize trial results $\mathcal{R} \in \mathbb{R}^{|L| \times |\widetilde{O}_i^k|}$ to a zero matrix}
\For{$i \gets 1 \  to \  2$}{
    \For{$j \gets   1 \ to \ |\widetilde{O}_i^k| $}{
        $\boldsymbol{\alpha}'_i \gets \widetilde{O}_i^k[j]$ \;
        Copy model $\theta'_i \gets \theta_i$ \;
        Train $\theta'_i$ on $D^{trial}$ for one epoch using teacher model $\theta_{3-i}$ and  $\boldsymbol{\alpha}'_i$ with~\myref{eq:Pareto-MD_Loss} \;
        \For{$\ell \gets 1 \ to\  |L|$}{
            $\mathcal{R}[\ell][j] \gets \mathcal{L}_{ce}(D^{valid}_{\ell}; \theta'_i)$
        }
    }

    \For{$\ell \gets 1 \ to\  |L|$}{
        $\hat{j}  \gets \mathop{\arg\min}\limits_{j} \mathcal{R}[\ell][j]$ \;
        $\boldsymbol{\alpha}^k_i[\ell] \gets \widetilde{O}_i^k[\hat{j}][\ell]$
    }
}
\end{algorithm}
Next, \textsc{Auto}-PMD uses each distillation weight vector in $\widetilde{O}^k$ to train the current model on trial set $D^{trial}$, which is constructed by sampling $\rho$ of $D^{train}$, for one epoch. 
 The three trained models are evaluated on the validation set, and the language-specific dev losses of these models form a matrix, which is represented by trial results $\mathcal{R} \in \mathbb{R}^{|\widetilde{O}^k|\times |L|}$. 
 The model training of this step incurs overhead, which is proportional to the value of $\rho \times |\widetilde{O}^k|$. In this work, we set $\rho=0.1$. Thereby, the extra overhead is 30\% of the actual model training.

Finally, the language-specific optimal actions are selected according to the trial results and then performed on $\boldsymbol{\alpha}^{k-1}[\ell]$, obtaining the results of $\boldsymbol{\alpha}^{k}[\ell]$. We exemplify this step with Fig.~\ref{fig:Auto-PMD}. The red model, trained using the increased version of $\boldsymbol{\alpha}^{k-1}$ (the vector in red), achieves the best performance of \textit{Fr$\rightarrow$En}. Thus, the $\boldsymbol{\alpha}^{k}[\ell]$ of \textit{Fr$\rightarrow$En} is obtained by increasing the $\boldsymbol{\alpha}^{k-1}[\ell]$ of \textit{Fr$\rightarrow$En}.

\paragraph{Implementation of Actions.} As aforementioned, three actions for updating distillation weights are defined (in \myref{eq: actions set}). The $f_{=}(\cdot)$ is simple:
\begin{equation}
    f_{=}(\boldsymbol{\alpha}[\ell]) = \boldsymbol{\alpha}[\ell].
\end{equation}
For $f_{\uparrow}(\cdot)$ and $f_{\downarrow}(\cdot)$, it must ensure that the output is always between $[0, 1]$. 
Therefore, the input is first mapped into $(-\infty, +\infty)$ using the inverse of sigmoid function and then increased/decreased the value by $\mu$, named step size.
Finally, the increased/decreased value is mapped back into $[0, 1]$ using sigmoid function. 
Formally:
\begin{equation}
f_{\uparrow}(\boldsymbol{\alpha}[\ell]) = \sigma(\sigma^{-1}(\boldsymbol{\alpha}[\ell]) + \mu) 
\end{equation}
\begin{equation}
f_{\downarrow}(\boldsymbol{\alpha}[\ell]) = \sigma(\sigma^{-1}(\boldsymbol{\alpha}[\ell]) - \mu)
\end{equation}
where $\sigma(\cdot)$ is sigmoid function. The step size $\mu$ is crucial for weights search. A smaller step size could improve the precision of searched weights while may delay convergence to the optimal weight.
Therefore, we design a \textbf{step size scheduler}, setting a large step size in the early training stage and then deducing the step size:
\begin{equation}
\label{eq:step size decay}
\mu = \sqrt{\frac{\mathcal{T}_{max} - t}{\mathcal{T}_{max}}}
\end{equation}
where $\mathcal{T}_{max}$ is the max training steps.

\begin{table*}[!t]
\small
\renewcommand\tabcolsep{4.0pt}
\renewcommand\arraystretch{1.15}
\centering
\begin{tabular}{lccccc}

\toprule
\multirow{2}{*}{\textbf{Method}} & \multirow{2}{*}{\makecell[c]{\textbf{Sampling}}} & \multicolumn{2}{c}{\textbf{\textsc{WMT-6}}} & \multicolumn{2}{c}{\textbf{\textsc{TED-8-Diverse}}} \\
\cmidrule(lr){3-4}
\cmidrule(lr){5-6}
 & & \textbf{Many-to-One} & \textbf{One-to-Many} & \textbf{Many-to-One} & \textbf{One-to-Many} \\

\midrule
\multicolumn{6}{c}{\textbf{Existing Balancing Training Strategies}} \\
\textsc{Temperature Sampling} & $\tau=1$ & 20.57 & 18.92 & 29.00 & 22.75 \\
 \textsc{Temperature Sampling} & $\tau>1$ & 19.93 & 18.63 & 28.35 & 22.23 \\
 \textsc{MultiDDS-S}~\cite{wang-etal-2020-balancing}$^*$ & \textit{dyn.} & -- & -- & 27.00 & 18.24  \\
\textsc{MultiUAT}~\cite{wu-etal-2021-uncertainty}$^*$ & \textit{dyn.} & -- & -- & 27.83 & 19.76  \\
CCL-M~\cite{zhang-etal-2021-competence-based}$^*$ & \textit{dyn.} & -- & -- & 28.34 & 19.53 \\
 $\chi$-IBR~\cite{zhou-etal-2021-distributionally}$^*$ & \textit{dyn.}
& -- & -- & 29.74 & 23.44  \\
 
\midrule
\multicolumn{6}{c}{\textbf{Existing Knowledge Distillation-based Strategies}} \\

\textsc{Multi-Distill}~\cite{tan-etal-2018-multilingual} & $\tau=1$ & 20.18 & 18.57 & 29.52 & 22.31   \\
 LSSD~\cite{huang-etal-2022-unifying} & $\tau=1$ & 21.17 & 19.76 & 30.77 & \textbf{23.55} \\
 
\midrule
\multicolumn{6}{c}{\textbf{Our Pareto Mutual Distillation}} \\

 \multirow{2}{*}{\textsc{Uni}-PMD} & $\tau=1$  &  20.76$^\dagger$  & 18.96  &  29.76$^\dagger$  & 22.92  \\
  & $\tau>1$  &  21.74$^\dagger$  &  19.76$^\dagger$ &  29.97$^\dagger$   &  22.91   \\ \cdashline{1-6}

 \multirow{2}{*}{\textsc{Bi}-PMD} & $\tau=1$  &   21.61$^\dagger$ &  19.53$^\dagger$  &  30.31$^\dagger$ &  23.00$^\dagger$ \\
  & $\tau>1$   &  21.92$^\dagger$ &  20.09$^\dagger$ &  30.42$^\dagger$   &  22.77  \\ \cdashline{1-6}

 \multirow{2}{*}{\textsc{Auto}-PMD} & $\tau=1$  &  21.89$^\dagger$   &  20.16$^\dagger$  & \textbf{31.05}$^\dagger$ &  23.31$^\dagger$  \\
 & $\tau>1$  &  \textbf{22.39}$^\dagger$  &   \textbf{20.48}$^\dagger$  & 30.71$^\dagger$  &  23.28$^\dagger$   \\
\bottomrule
\end{tabular}
\caption{BLEU scores on the WMT-6 and TED-8-Diverse dataset.
Bold indicates the highest BLEU score in each setting. 
`*' means results taken from the original paper.
`$\dagger$' indicates significantly better than temperature-based sampling with t-test $p<0.001$.
The temperature-based sampling is tried with $\tau=\{1, 5\}$ on WMT-6 and $\tau=\{1, 3\}$ on TED-8-Diverse.
For each of our approaches, the first row is the result of model-1, and the second row is the result of model-2.
`\textit{dyn.}' is the abbreviation for ``dynamic sampling.''
 }
\label{tab:main results}
\end{table*}
\section{Experiments}
\subsection{Settings}

\paragraph{Datasets.}
We conduct experiments on two datasets: the \text{WMT-6} dataset provided by~\citet{huang-etal-2022-unifying} and the widely-used TED-8-Diverse dataset constructed by~\citet{wang-etal-2020-balancing}. 
The WMT-6 dataset involves the language pairs of 3 LRLs (\textit{et}, \textit{ro}, \textit{tr}) and 3 HRLs (\textit{fr}, \textit{de}, \textit{zh}) to English.
This dataset has around 5M training sentences from parallel corpora that WMT provides over multiple years, and the corresponding validation and test sets are used. The data statistics are detailed in Appendix~\ref{section:data statistics}.
The TED-8-Diverse contains the language pairs of 4 LRLs (\textit{bos}, \textit{mar}, \textit{hin}, \textit{mkd}) and 4 HRLs (\textit{ell}, \textit{bul}, \textit{fra}, \textit{kor}) to English. 
This dataset comprises around 570K sentence pairs.
The data statistics and the interpretation of language codes are demonstrated in Appendix~\ref{section:data statistics}.
Compared to TED-8-Diverse, the size of WMT-6 dataset is more considerable and distributed more unevenly.

For each dataset, our approach is evaluated in two multilingual translation scenarios: 1) \textsc{Many-to-One}~(M2O): translating multiple languages to English in this work; 2) \textsc{One-to-Many}~(O2M): translating English to other languages. 

\paragraph{Hyper-parameters.}
Even though our proposed training framework can be applied to any model architecture, we verify its effectiveness on the popular Transformer~\cite{vaswani-etal-2017-attention} implemented in fairseq~\cite{ott-etal-2019-fairseq} with the base version. We use the same model configuration, hyper-parameters, and preprocess procedure as those of~\citet{huang-etal-2022-unifying} for all baselines and our method. 
The only difference is that the dropout rate is modified into $0.2$ on WMT-6, to accelerate the convergence without performance loss.
The complete set of hyper-parameters is demonstrated in Appendix~\ref{subsection:hyper-parameters}.
The performance is evaluated with the BLEU score~\cite{papineni-etal-2002-bleu} using the SacreBLEU toolkit~\cite{post-2018-call}.

As illustrated in \S\ref{subsection:method framework}, our Pareto-MD trains two models using different sampling distributions, $P_1$ and $P_2$, and we adopt temperature-based sampling with different values of $\tau$ to produce these two distributions.
We set $\tau=1$ for $P_1$ and $\tau=5$ for $P_2$ on WMT-6. On TED-8-Diverse, we set $\tau=1$ for model-1 and $\tau=3$ for model-2 since an overly large value leads to poor performance.
For the \textsc{Uni}-PMD and \textsc{Bi}-PMD, we manually search the optimal $\alpha$ (in \myref{eq:uni-momd} and \myref{eq:bi-momd}) among $\{ 0.2, 0.4, 0.6, 0.8\}$.
The update interval of distillation weights $\mathcal{T}$ is set to the step number of one epoch.

\paragraph{Baselines.}
We primarily compare our Pareto-MD with: 
(1)~Temperature-based Sampling: the method most related to our work; 
2)~$\chi$-IBR~\cite{zhou-etal-2021-distributionally}, the state-of-the-art (SOTA) dynamic sampling method, which enables the balancing training based on \textit{distributionally robust optimization};
3)~LSSD~\citep{huang-etal-2022-unifying}, another distillation-based training strategy which achieves SOTA performance on TED-8-Diverse and WMT-6 dataset via alleviating the convergence inconsistency problem of MNMT using self-distillation. 
More details of baselines are demonstrated in Appendix~\ref{subsection:details about re-implementing baselines}.

\subsection{Main Results} 
We summarize the main results in Table~\ref{tab:main results}. As we observed, our methods significantly outperform the temperature-based sampling under M2O and O2M settings on both datasets. 
The model-2 trained with \textsc{Auto}-PMD has improved by up to +2.46 BLEU under the M2O setting of WMT-6.
Furthermore, Pareto-MD achieves higher BLEU scores than previous methods in most settings.
At best, \textsc{Auto}-PMD outperforms the previous SOTA (LSSD) by +1.22 BLEU scores under the M2O setting of WMT-6.
When comparing \textsc{Uni}-PMD and \textsc{Bi}-PMD, it is obvious that \textsc{Bi}-PMD consistently exceeds \textsc{Uni}-PMD, verifying the motivation that the worse model is also possible to improve the better model via knowledge distillation.
\textsc{Auto}-PMD further surpasses \textsc{Bi}-PMD by +0.3$\sim$0.5 BLEU. This improvement proves that our automatic search of distillation weights is indeed reliable. Moreover, \textsc{Auto}-PMD is more general than \textsc{Uni}-PMD and \textsc{Bi}-PMD since it eliminates the need to search for the hyper-parameter $\alpha$ manually\footnote{The effect of $\alpha$ is shown in Appendix~\ref{section:effect of differrent alpha}.}.

\begin{figure}[t]
  \centering
    \includegraphics[clip,width=1.0\columnwidth,]{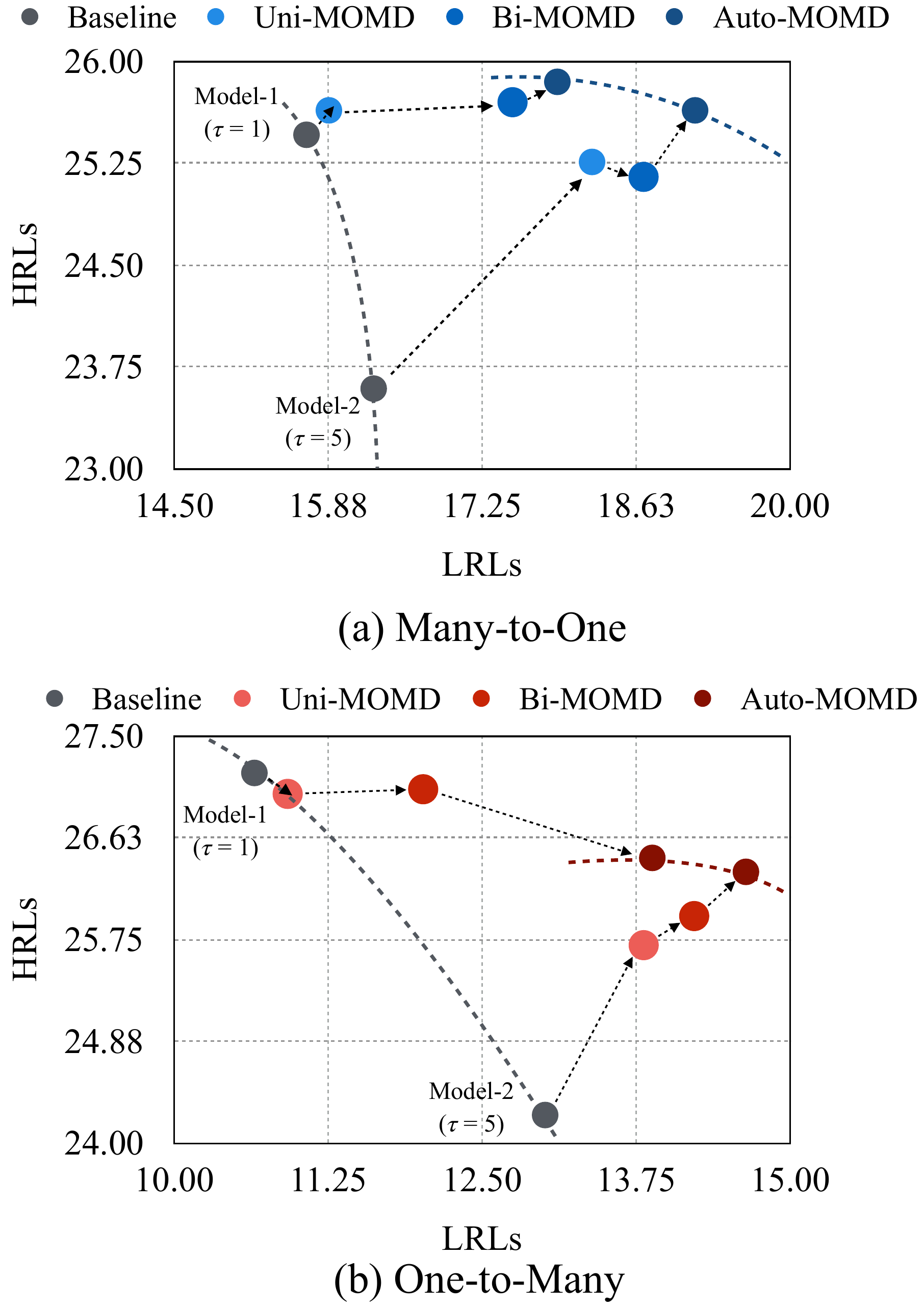}
    \caption{Multilingual performance Pareto frontier on the WMT-6 dataset. Gray dotted curves indicate the Pareto frontier of baselines and the colorful ones mark the frontier made by \textsc{Auto}-PMD. This figure shows that the Pareto frontier is pushed outwards significantly.}
  \label{fig:visualization of multilingual performance Pareto Frontier}
\end{figure}
\begin{table}[ht]
\renewcommand\tabcolsep{20.0pt} 
\small
\centering
\begin{tabular}{lcc}
\toprule
\textbf{Method} & \textbf{Sampling} & \textbf{BLEU} \\
\midrule
\multirow{2}{*}{Vanilla MD} & $\tau=1$ & 20.93 \\
  & $\tau=1$ & 20.97  \\
  \midrule
  \multirow{2}{*}{Vanilla MD} & $\tau=5$ & 21.13  \\
  & $\tau=5$ & 21.29   \\
\midrule

\multirow{2}{*}{\textsc{Bi}-PMD} & $\tau=1$ & 21.61  \\
  & $\tau=5$ & 21.92   \\
  \midrule
  \multirow{2}{*}{\textsc{Auto}-PMD} & $\tau=1$ & 21.89  \\
  & $\tau=5$ & \textbf{22.39}   \\
\bottomrule
\end{tabular}
\caption{Comparison between our method with vanilla mutual distillation (Vanilla MD) under the Many-to-One setting of the WMT-6 dataset.}
\label{table:comparison with vanilla MD}
\end{table}
\section{Analysis}
\subsection{Visualization of Pareto Frontier}
In order to clearly assess the impact of our methods on HRLs and LRLs, we visualize the Pareto frontier in Fig.~\ref{fig:visualization of multilingual performance Pareto Frontier}.
Three important observations can be drawn: 
1) overall, the model-1 has been significantly shifted right, and the model-2 has been shifted upwards, proving that Pareto-MD effectively alleviates the shortcomings of each model as we expected; 
2) both of model-1 and model-2 are shifted right beyond the original model-2, indicating that the performance of LRLs is improved beyond the original performance bound.
The reason may be that the transfer learning from HRLs to LRLs is more effective when the model achieves high performance on both HRLs and LRLs;
3) the model-1 degenerates on the translation of HRLs in the O2M setting. One potential cause is the representation space of HRLs undergoing more intense squeezing in the O2M than in the M2O when the model learns well on LRLs.
\subsection{Effect of Diverse Sampling Strategies}
In the Pareto-MD training framework, two models corresponding to different Pareto optimal solutions are trained collaboratively using distinct training distributions. One natural question that arises is, how would the performance be affected if we trained two models with the same training distribution? This approach, in fact, degenerates into the vanilla mutual distillation method. 
Therefore, we conduct a comparison experiment on the WMT-6 dataset (M2O setting) shown in Table~\ref{table:comparison with vanilla MD}.
The results indicate that vanilla mutual distillation underperforms our \textsc{Bi}-PMD by about 0.6 BLEU, which supports the effectiveness of using different sampling distributions for our Pareto-MD.
Moreover, we propose \textsc{Auto}-PMD to improve vanilla mutual distillation by +1.1 BLEU totally.

\begin{figure}[t]
  \centering
    \includegraphics[clip,width=1.0\columnwidth,]{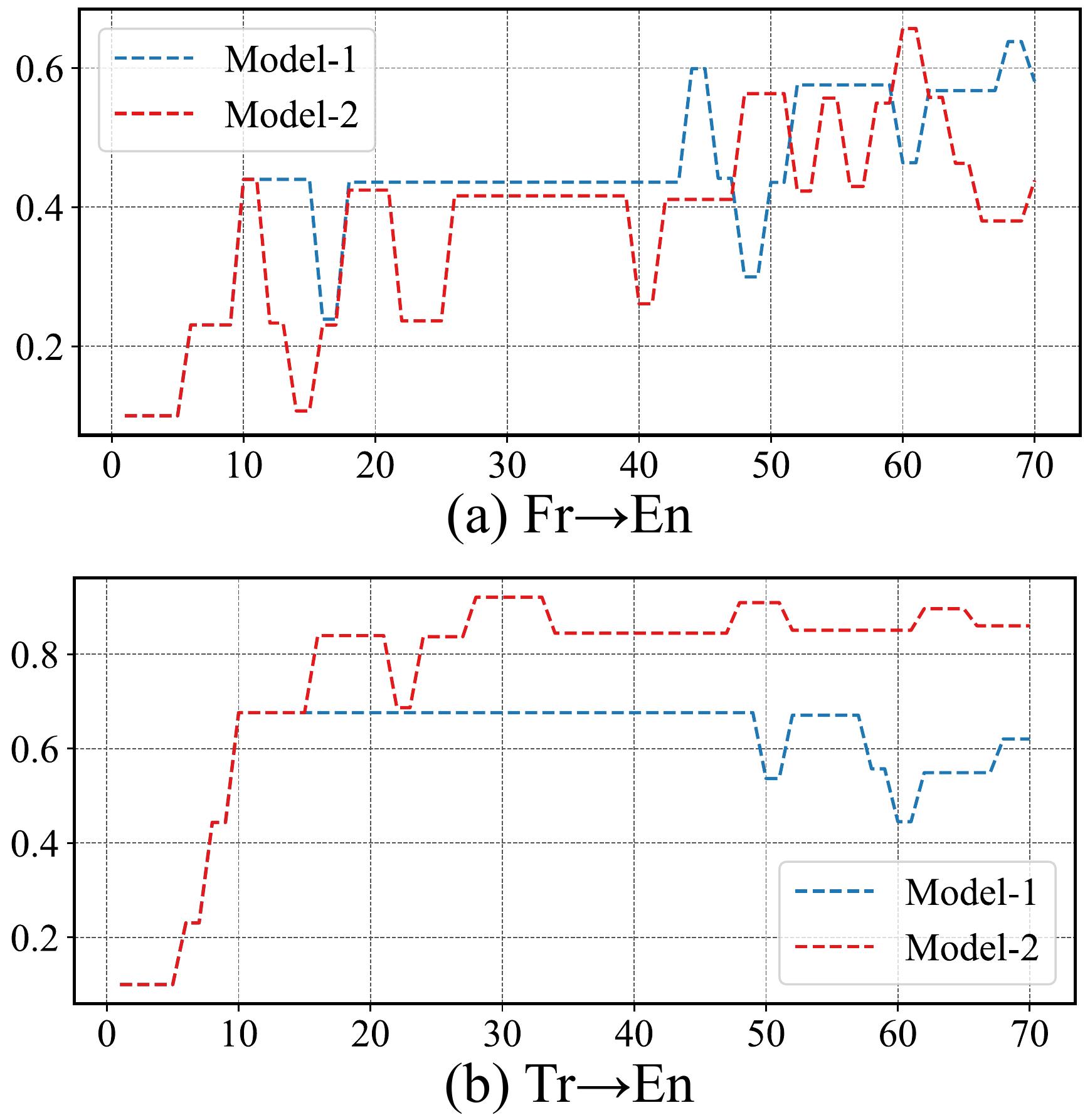}
    \caption{Visualization of automatically search distillation weights in the many-to-one setting of WMT-6 dataset. Due to the space limitation, we only show the weights of one HRL (Fr$\rightarrow$En) and one LRL (Tr$\rightarrow$En)}
  \label{fig:visualization of searched distillation weights}
\end{figure}
\subsection{Evolution of Distillation Weights}
To better understand the process of \textsc{Auto}-PMD, we visualize the automatically searched distillation weights in Fig.~\ref{fig:visualization of searched distillation weights}.
As it depicts, the distillation weights constantly vary to adapt the dynamic models with a decreasing variance made by the decay of search step size (\myref{eq:step size decay}).
Besides, it is discovered that the low-resource Tr$\rightarrow$En translation favors a higher value of distillation weight than the high-resource Fr$\rightarrow$En translation.
This phenomenon makes sense since LRLs suffer from more serious over-fitting~\citep{huang-etal-2022-unifying}, requiring stronger distillation learning.

\subsection{Effect of Step Size Scheduler $\mu$}
\label{section:effect of step size scheduler}
The performance of different step size schedulers is listed in Table~\ref{tab:effect of step size shceduler.}. The simple scheduler-1 fixes the step size to $1.0$, performing relatively poorly.
The scheduler-2 decreases the step size from $1.0$ to $0.2$.
The scheduler-4 decreases the step size from $1.0$ to $0.0$, achieving the best performance.
The scheduler-3 also decrease the step size from $1.0$ to $0.0$, while not performing searching of distillation weights at the end of training.
We finally adopt the scheduler-4 in our \textsc{Auto}-PMD.

\section{Related Work}
\begin{table}[t]

\renewcommand\tabcolsep{4.0pt}
\renewcommand\arraystretch{2}
\centering
\begin{tabular}{llc}
\toprule
\textbf{\ \#} & \makecell[c]{\textbf{Scheduler}} & \makecell[c]{\textbf{BLEU} \\ ($\tau=1$ / $\tau=5$)} \\
\midrule
\ 1 & \ $\mu = 1$ & 20.71 / 21.80 \\
\midrule
\ 2 & \ $\mu = \sqrt{\frac{\mathcal{T}_{max} - 0.8 \times t}{\mathcal{T}_{max}}}$ & 21.90 / 22.21 \\
\midrule
\ 3 & \ $\mu = max(\sqrt{\frac{\mathcal{T}_{max} - 1.2 \times t}{\mathcal{T}_{max}}}, 0)$ & 21.74 / 22.31 \\
\midrule
\ 4 & \ $\mu = \sqrt{\frac{\mathcal{T}_{max} - t}{\mathcal{T}_{max}}}$ & 21.89 / \textbf{22.39} \\
\bottomrule
\end{tabular}
\caption{Effect of step size scheduler $\mu$ in the many-to-one translation of WMT-6 dataset.
We have tried for four implementations of the step size scheduler.}
\label{tab:effect of step size shceduler.}
\end{table}
For a long time, data imbalance has been a problem hindering multilingual models from performing evenly across different languages.
Existing methods pursue balanced performance via designing heuristics~\citep{arivazhagan-etal-2019-massively} or automatic sampling strategies~\citep{arivazhagan-etal-2019-massively,wang-etal-2020-balancing,zhou-etal-2021-distributionally,wu-etal-2021-uncertainty,zhang-etal-2021-competence-based}. For example, \citet{wang-etal-2020-balancing} design a Reinforce Learning based method to automatically adjust the sampling probability of each language pair towards an overall optimal solution.
\citet{zhou-etal-2021-distributionally} vary the distribution via distributional robust optimization.
However, their improvement is limited since increasing the training weights of some languages leads to relative decreases in the weights of other languages, resulting in a trade-off on the Pareto frontier.
Different from their methods, we overcome this issue by training two models collaboratively.

Before our work, there were two approaches also based on knowledge distillation in MNMT.
 \citet{tan-etal-2018-multilingual} use pre-defined bilingual models to teach the multilingual model via knowledge distillation.
\citet{huang-etal-2022-unifying} propose language-specific self-distillation to remedy the convergence inconsistency problem in MNMT using self-distillation.
Our Pareto-MD is an extension of mutual distillation on the Pareto optimization problems.
\section{Conclusion}
In this work, we propose a training framework Pareto-MD to reach a higher Pareto frontier for MNMT.
The core of Pareto-MD is the synergy between diverse Pareto optimal solutions via mutual distillation.
Besides, we design a novel strategy for deducing distillation weights automatically, achieving better performance and getting rid of hyper-parameter searching.
Experimental results on the WMT and TED datasets show the effectiveness of our method.
Even though we experiment with training two models in this work, our method can naturally apply to train more models.
In the future, we are interested in exploring how to apply our Pareto-MD to the training of large language models~\citep{zhao2023survey}.

\section*{Limitations}
Our Pareto-MD doubles computational cost due to training two models simultaneously, which can be a limitation of our approach.
However, Pareto-MD obtains significant improvement that is hard to achieve for previous methods of training individual models, thus worthy.
Besides, our approach would not necessarily result in double training time because these two models can be trained in parallel as implemented by~\citet{guo2020online}.
Moreover, Pareto-MD does not affect inference efficiency.

\section*{Acknowledgements}
Xiaocheng Feng is the corresponding author of this work. We thank the anonymous reviewers for their insightful comments. This work was supported by the National Key R\&D Program of China via grant 2020AAA0106502, National Natural Science Foundation of China (NSFC) via grant 62276078, the Key R\&D Program of Heilongjiang via grant 2022ZX01A32 and the International Cooperation Project of PCL, PCL2022D01.

\bibliographystyle{acl_natbib}
\bibliography{anthology,custom}

\clearpage
\appendix
\onecolumn
\section{Equivalence Between Searching in $O^k$ and $\widetilde{O}^k$}
\label{section:proof of equivalence}
As illustrated in \S\ref{subsection:auto-momd}, our strategy \textsc{Auto}-PMD first searches the language-specific optimal multilingual distillation weight vector $\hat{\boldsymbol{\alpha}}^\ell$ for each translation direction $\ell$ from a search space and then take the $\hat{\boldsymbol{\alpha}}^\ell[\ell]$ as the searching result of $\boldsymbol{\alpha}^k[\ell]$.
To search the optimal solution, the search space should be the entire space $O^k$, which is formalized as:
\begin{equation}
\nonumber
O^k = \mathop{\bigtimes}\limits_{\ell \in L} \{ f(\boldsymbol{\alpha}^{k-1}[\ell]) \,|\, f \in \mathcal{F} \},
\end{equation}
However, the size of $O^k$ grows exponentially as the number of languages increases. 
Therefore, we instead search in $\widetilde{O}^k$, a subset of $O^k$, which is formalized as:
\begin{equation}
\begin{aligned}
\nonumber
\widetilde{O}^k =
\{\  
&\{\, f_{\uparrow}(\boldsymbol{\alpha}^{k-1}[\ell]) \,\}_{\ell \in L}, \\
&\{\, f_{\downarrow}(\boldsymbol{\alpha}^{k-1}[\ell]) \,\}_{\ell \in L}, \\
&\{\, f_{=}(\boldsymbol{\alpha}^{k-1}[\ell]) \,\}_{\ell \in L} 
\  \}
.
\end{aligned}
\end{equation}
In this section, we initially give a formal definition of the searching process. Subsequently, the Distillation Weights Independence (DWI) assumption is introduced. Ultimately, we prove the equivalence between searching in $O^k$ and $\widetilde{O}^k$ based on the DWI assumption.
\begin{definition}[Searching Process]
    Given the multilingual trial set $D^{trial} = \{ D^{trial}_{\ell} \}_{\ell=1}^{|L|}$, validation set $D^{valid} = \{ D^{valid}_{\ell} \}_{\ell=1}^{|L|}$, student mode $\theta^S$, teacher model $\theta^T$, and the search space $O$ , for each translation direction $\ell$, the searching process of $\boldsymbol{\alpha}^k[\ell]$ is:
    \begin{equation}
    \begin{aligned}
    \nonumber
        \boldsymbol{\alpha}^k[\ell] &= \hat{\boldsymbol{\alpha}}^\ell[\ell] \\
        \hat{\boldsymbol{\alpha}}^\ell \ \  &= \mathop{\arg\min}\limits_{\boldsymbol{\alpha} \in O} \mathcal{L}_{ce}(D^{valid}_{\ell}; \hat{\theta}(\boldsymbol{\alpha})) \\
        \hat{\theta}(\boldsymbol{\alpha}) \  &= \mathop{\arg\min}\limits_{\theta}\mathcal{L}_{PMD}(D^{trial}; \theta^S, \theta^T, \boldsymbol{\alpha}).
    \end{aligned}
    \end{equation}
\end{definition}

\begin{hypothesis}[Distillation Weights Independence]
\label{hyp:Distillation Weights Independence}
Given two multilingual distillation weight vectors $\boldsymbol{\alpha}_1$ and $\boldsymbol{\alpha}_2$: 
\begin{equation}
\begin{aligned}
    \nonumber
\exists \ell \in L, \boldsymbol{\alpha}_1[\ell] &= \boldsymbol{\alpha}_2[\ell] \\
\Rightarrow
    \mathcal{L}_{ce}(D^{valid}_{\ell}; \hat{\theta}(\boldsymbol{\alpha}_1)) &= \mathcal{L}_{ce}(D^{valid}_{\ell}; \hat{\theta}(\boldsymbol{\alpha}_2))
\end{aligned}
\end{equation}
\end{hypothesis}

\begin{theorem}

    Let $\hat{\boldsymbol{\alpha}}^\ell[\ell]$ denote the searching result in the search space $O^k$ for direction $\ell$, $\widetilde{\boldsymbol{\alpha}}^\ell[\ell]$ denotes the searching result in the search space $\widetilde{O}^k$ for direction $\ell$, based on the Distillation Weights Independence assumption, it is satisfied that:
    \begin{equation}
    \nonumber
        \hat{\boldsymbol{\alpha}}^\ell[\ell] = \widetilde{\boldsymbol{\alpha}}^\ell[\ell].
    \end{equation}    
\end{theorem}

\begin{proof}
    Let $\hat{\boldsymbol{\alpha}}^\ell[\ell] =  \hat{f}^l(\boldsymbol{\alpha}^{k-1}[\ell])$, where $\hat{f}^l \in \mathcal{F}$ is the language-specific action, the following equation holds:
    \begin{equation}
        \nonumber
        \mathcal{L}_{ce}(D^{valid}_{\ell}; \theta(\hat{\boldsymbol{\alpha}}^\ell)) = \mathcal{L}_{ce}(D^{valid}_{\ell}; \theta(\{ \hat{f}^l(\boldsymbol{\alpha}^{k-1}[\ell']) \}_{\ell' \in L})),
    \end{equation}
    based on hypothesis~\ref{hyp:Distillation Weights Independence}.
    Because $\{ \hat{f}^l(\boldsymbol{\alpha}^{k-1}[\ell']) \}_{\ell' \in L} \in \widetilde{O}^k$, and $\widetilde{O}^k \subseteq O^k$, then we can infer that:
    \begin{equation}
    \nonumber
    \begin{aligned}
        &\Rightarrow&
        \mathcal{L}_{ce}(D^{valid}_{\ell}; \{ \hat{f}^l(\boldsymbol{\alpha}^{k-1}[\ell']) \}_{\ell' \in L}) &= \min\limits_{\boldsymbol{\alpha} \in \widetilde{O}^k} \mathcal{L}_{ce}(D^{valid}_{\ell}; \hat{\theta}(\boldsymbol{\alpha})) \\
        &\Rightarrow&
        \{ \hat{f}^l(\boldsymbol{\alpha}^{k-1}[\ell']) \}_{\ell' \in L} &= \mathop{\arg\min}\limits_{\boldsymbol{\alpha} \in \widetilde{O}^k} \mathcal{L}_{ce}(D^{valid}_{\ell}; \hat{\theta}(\boldsymbol{\alpha})) \\
        &\Rightarrow&
        \hat{f}^l(\boldsymbol{\alpha}^{k-1}[\ell]) &= \widetilde{\boldsymbol{\alpha}}^\ell[\ell] \\
        &\Rightarrow&
        \hat{\boldsymbol{\alpha}}^\ell[\ell] &= \widetilde{\boldsymbol{\alpha}}^\ell[\ell]
    \end{aligned}
    \end{equation}
\end{proof}

\twocolumn
\section{Data Statistics}
\label{section:data statistics}
We list data statistic of TED-8-Diverse dataset in Table~\ref{tab:data statistics.}.
Data statistics of WMT-6 dataset is listed in Table~\ref{tab:data statistics WMT.}.
\begin{table}[th]
\centering

\begin{tabular}{lr}
\toprule
\textbf{Language} & \textbf{Num} \\
\midrule
\textit{bos} (Bosnian) & 5,664  \\

\textit{mar} (Marathi) & 9,840  \\

\textit{hin} (Hindi) & 18,798  \\

\textit{mkd} (Macedonian) & 25,335 \\

\textit{ell} (Greek) & 134,327 \\

\textit{bul} (Bulgarian) & 174,444  \\

\textit{fra} (French) & 192,304  \\

\textit{kor} (Korean) & 205,640 \\
\bottomrule
\end{tabular}
\caption{Data statistics for the TED-8-Diverse dataset. `num' refers to the number of sentence pairs in the training set.}
\label{tab:data statistics.}
\end{table}

\begin{table}[th]

\renewcommand\tabcolsep{3.0pt}
\centering

\begin{tabular}{l|cr}
\toprule
\textbf{Language} & \textbf{Data Source} & \textbf{Num} \\
\midrule

\textit{tr} (Turkish) & WMT17 & 5,000 \\

\textit{ro} (Romanian) & WMT16 & 10,000 \\

\textit{et} (Estonian) & WMT18 & 80,000 \\

\textit{zh} (Chinese) & WMT17 & 400,000 \\

\textit{de} (German) & WMT14 & 1,500,000 \\

\textit{fr} (French) & WMT14 & 3,000,000 \\

\bottomrule
\end{tabular}
\caption{Data statistics for the WMT dataset. `num' refers to the number of sentence pairs in the training set.}
\label{tab:data statistics WMT.}
\end{table}

\begin{table*}[t]
\renewcommand\tabcolsep{4.0pt}
\linespread{1.5}
\small
    \centering
    \begin{tabular}{c | c  c  cccccc c}
    \toprule
        
        \textbf{Setting} & \textbf{Method} & \textbf{Sampling} & \textbf{fr} & \textbf{de} & \textbf{zh} & \textbf{et} & \textbf{ro} & \textbf{tr} & \textbf{\textit{Avg.}} \\
        \cline{1-10}
    
        \multirow{4}{*}{M2O} &
        \multirow{2}{*}{\textit{Temperature Sampling}} & $\tau=1$ & 34.40 & 28.70 & 13.27 & 16.41 & 22.65 & 7.99 & 20.57 \\
        & & $\tau>1$ & 31.59 & 26.61 & 12.56 & 16.48 & 23.06 & 9.29 & 19.93 \\

        \cdashline{2-10}

        & \multirow{2}{*}{\textsc{Auto}-PMD} & $\tau=1$ & \textbf{34.96} & \textbf{28.79} & 13.81 & 17.9 & 25.22 & 10.65 & 21.89 \\
        & & $\tau>1$ & 34.09 & 28.77 & \textbf{14.05} & \textbf{19.22} & \textbf{26.62} & \textbf{11.60} & \textbf{22.39} \\
    
        \cline{1-10}
        \multirow{4}{*}{O2M} &
        \multirow{2}{*}{\textit{Temperature Sampling}} & $\tau=1$ & \textbf{36.16} & \textbf{23.89} & 21.49 & 11.53 & 14.85 & 5.58 & 18.92  \\
        & & $\tau>1$ & 31.21 & 20.76 & 20.76 & 13.28 & 17.54 & 8.20 & 18.63  \\
        \cdashline{2-10}

        & \multirow{2}{*}{\textsc{Auto}-PMD} & $\tau=1$ & 35.38 & 23.12 & 20.84 & 13.2 & 18.79 & 9.65 & 20.16  \\
        & & $\tau>1$ & 34.47 & 23.00 & \textbf{21.51} & \textbf{14.15} & \textbf{19.54} & \textbf{10.23}  & \textbf{20.48}  \\
    \bottomrule
    \end{tabular}
    \caption{BLEU score per language pair on the WMT-6 dataset.
    `\textit{Avg.}' is the abbreviation of ``average values''.
    Bold indicates the best performance of each language pair. 
    Languages are sorted in decreasing order from left to right according to data size.}
    \label{tab:results per languages WMT-6}
    \vspace{-3mm}
\end{table*}
\begin{table*}[t]
\renewcommand\tabcolsep{4.0pt}
\linespread{1.5}
\small
    \centering
    \begin{tabular}{c | c  c  cccccccc c}
    \toprule
        \textbf{Setting} & \textbf{Method} & \textbf{Sampling} & \textbf{kor} & \textbf{fra} & \textbf{bul} & \textbf{ell} & \textbf{mkd} & \textbf{hin} & \textbf{mar}   & \textbf{bos} & \textbf{\textit{Avg.}} \\
        \cline{1-12}
    
        \multirow{4}{*}{M2O} &
        \multirow{2}{*}{\textit{Temperature Sampling}} & $\tau=1$ & 19.73 & 40.73 & 39.74 & 38.71 & 34.34 & 23.38 & 11.13 & 24.88 & 29.08 \\
        & & $\tau>1$ & 18.79 & 40.1 & 39.00 & 38.11 & 32.89 & 22.55 & 10.36 & 24.98 & 28.35 \\

        \cdashline{2-12}

        & \multirow{2}{*}{\textsc{Auto}-PMD} & $\tau=1$ & \textbf{21.14} & \textbf{42.41} & \textbf{41.52} & \textbf{40.67} & \textbf{36.49} & \textbf{25.9} & 12.32 & 27.94 & \textbf{31.05} \\
        & & $\tau>1$ & 20.51 & 42.03 & 40.93 & 40.00 & 36.04 & 25.71 & \textbf{12.44} & \textbf{28.02} & 30.71 \\
    
        \cline{1-12}
        \multirow{4}{*}{O2M} &
        \multirow{2}{*}{\textit{Temperature Sampling}} & $\tau=1$ & 9.06 & 40.26 & 36.10 & 33.63 & 25.67 & 15.56 & 4.90 & 16.82 & 22.75 \\
        & & $\tau>1$ & 8.87 & 39.96 & 35.91 & 33.31 & 24.35 & 14.81 & 4.75 & 15.87 & 22.23 \\
        \cdashline{2-12}

        & \multirow{2}{*}{\textsc{Auto}-PMD} & $\tau=1$ & \textbf{9.13} & \textbf{40.94} & \textbf{36.56} & \textbf{34.03} & 27.15 & 15.89 & \textbf{5.13} & 17.64 & \textbf{23.31} \\
        & & $\tau>1$ & 8.90 & 40.65 & 36.55 & 33.64 & \textbf{27.44} & \textbf{16.29} & 4.90 & \textbf{17.89} & 23.28 \\
    \bottomrule
    \end{tabular}
    \caption{BLEU score per language pair on the TED-8-\textsc{Diverse} dataset.
    `\textit{Avg.}' is the abbreviation of ``average values''.
    Bold indicates the best performance of each language pair. 
    Languages are sorted in decreasing order from left to right according to data size.}
    \label{tab:results per languages ted-diverse}
    \vspace{-3mm}
\end{table*}

\section{\label{subsection:hyper-parameters}Hyper-parameters}
In this section, we report the hyper-parameters used in our experiments.
\begin{itemize}
    \item We adopt the base-version of Transformer architecture with 6 layers encoders/decoders and 8 attention heads.
    \item The embedding dimension is 512 and the Feed-Forward Network has a dimension of 2048.
    \item We train models with learning rate $\eta=0.0015$ and use Adam optimizer~\citep{DBLP:journals/corr/KingmaB14} with $\beta_1=0.9,\beta_2=0.98$, and the same learning rate schedule as \newcite{vaswani-etal-2017-attention}.
    \item Batch size is set to 64K and half-precision training is adopted~\cite{ott-etal-2018-scaling}. 
    \item For regularization, we use the label smoothing as 0.1~\cite{szegedy2016rethinking}. 
    We set the dropout as 0.3~\cite{JMLR:v15:srivastava14a} on the TED-8-Diverse dataset and as 0.2 on the WMT-6 dataset.
    \item Models are trained for 70 epochs on WMT-6 and 300 epochs on TED-8-Diverse according to the convergence.
    \item For TED-8-Diverse, we preprocess sentececes using sentencepiece~\citep{kudo-richardson-2018-sentencepiece} with a vocabulary size of 8\textit{K} for each language. For WMT-6, the vocabulary size is 64\textit{K} for all languages.
    \item For inference, we use beam search with beam size 5.
\end{itemize}
All models are trained on Tesla V100 GPUs.

\begin{figure}[t]
  \centering
    \includegraphics[clip,width=1.0\columnwidth,]{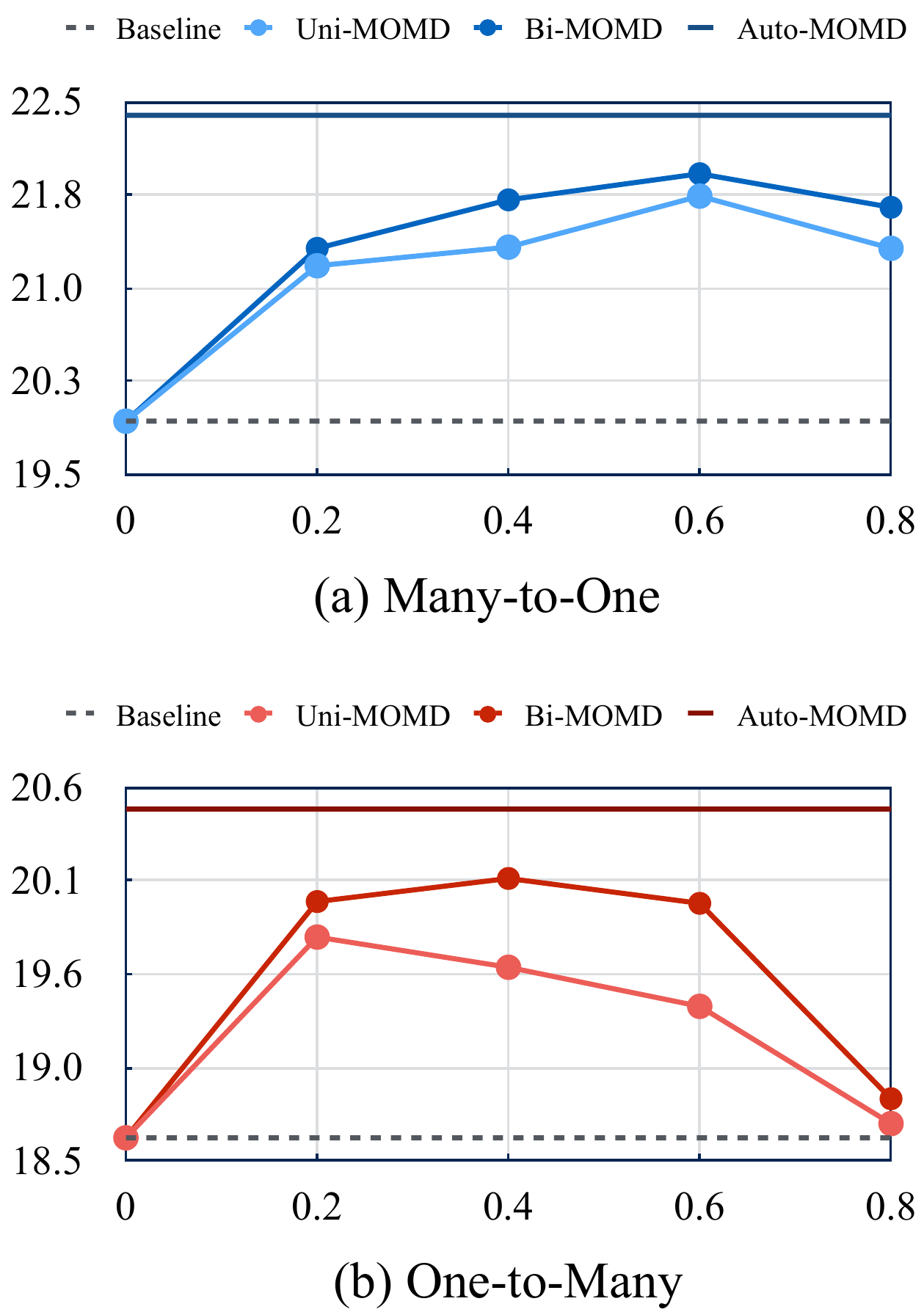}
    \caption{Effect of different values of $\alpha$ on WMT-6 dataset. For clarity, we only depict the results of model-2 trained with $\tau=5$.}
  \label{fig:effect of different alpha}
\end{figure}
\section{Details about Baselines \label{subsection:details about re-implementing baselines}}
For temperature-based sampling~\citep{arivazhagan-etal-2019-massively}, we adopt the official implementation in fairseq.
LSSD is re-implemented successfully with the code released by~\citet{huang-etal-2022-unifying}. We have tried to set Dropout rate to $\{0.2, 0.3\}$ for LSSD, and report the best results in terms of BLEU for fair comparison.
The code of $\chi$-IBR~\cite{zhou-etal-2021-distributionally} is also released. However, the result of $\chi$-IBR evaluated in our experiments is lower than the original paper. Therefore, we report the results in the original paper.

\section{BLEU scores on Individual Languages}
In this section, we report the BLEU scores of individual language pairs. For clarity, we only show the results of the temperature-based sampling and our \textsc{Auto}-PMD. As illustrated in Table.~\ref{tab:results per languages WMT-6} and Table.~\ref{tab:results per languages ted-diverse}, our method achieves consistent improvements in 3 out of 4 settings.

In the one-to-many setting of WMT-6 dataset, the performance of HRLs (i.e., \textit{fr} and \textit{de}) drops about 0.7 BLEU. This may be due to the parameter interference from the significantly improved LRLs.

\section{Effect of $\alpha$ for \textsc{Uni}-PMD and \textsc{Bi}-PMD}
\label{section:effect of differrent alpha}
In this section, we show the experimental results of \textsc{Uni}-PMD and \textsc{Bi}-PMD with different values of $\alpha$ in Fig.~\ref{fig:effect of different alpha}. 
As demonstrated, the value of $\alpha$ is crucial for the performance.
The optimal value of $\alpha$ varies across different settings.
This conclusion is consistent with former work related to knowledge distillation~\citep{huang-etal-2022-unifying}, which highlights the importance of deducing distillation weights automatically.

\section{Other Variants of Mutual Distillation}
\begin{table}[!t]
\renewcommand\tabcolsep{6.0pt}
\small
\centering
\begin{tabular}{lccc}
\toprule
\multirow{2}{*}{\textbf{Method}} & \multirow{2}{*}{\makecell[c]{\textbf{Sampling}}} & \multicolumn{2}{c}{\textbf{\textsc{BLEU}}}\\
\cmidrule{3-4}
& & \textbf{M2O} & \textbf{O2M} \\
\midrule

  \multirow{2}{*}{\textsc{Auto}-PMD} & $\tau=1$ & 21.89  & 20.16 \\
  & $\tau=5$ & \textbf{22.39} & \textbf{20.48}  \\
\midrule
  \multirow{2}{*}{\makecell[l]{\textsc{Dynamic}-MD }} & $\tau=1$ & 22.06   & 20.33  \\
  & $\tau=5$ & 22.11  & 20.24   \\
\midrule
  \multirow{2}{*}{\makecell[l]{LSMD}} & $\tau=1$ & 21.47   & 18.94   \\
  & $\tau=5$ & 21.03  & 19.46    \\
\bottomrule
\end{tabular}
\caption{Other variants of mutual distillation designed by us. 
\textsc{Dynamic}-MD is the abbreviation of Dynamic Mutual Distillation.
LSMD is the abbreviation of Language-Specific Mutual Distillation.}
\label{table:ablation study of Auto-PMD.}
\end{table}
In this work, we design another two mutual distillation-based strategies beyond \textsc{Auto}-PMD: Dynamic Mutual Distillation (\textsc{Dynamic}-MD) and Language-Specific Mutual Distillation (LSMD).
\textsc{Dynamic}-MD adopts the same update process of distillation weights as \textsc{Auto}-PMD. That is, \textsc{Dynamic}-MD also makes three trials and uses the optimal action to uptate the distillation weight. Differently, \textsc{Dynamic}-MD selects a uniform optimal action instead of language-specific optimal actions.
LSMD sets fixed and language-specific distillation weights for each language pair.
To obtain suitable language-specific distillation weights, we use the distillation weights searched by \textsc{Auto}-PMD at the last update.
The results of these two strategies are listed in Table~\ref{table:ablation study of Auto-PMD.}.
As the results show, \textsc{Auto}-PMD achieves higher performance upper-bound than these two strategies.














\end{document}